\documentclass[preprint,12pt]{elsarticle}

\usepackage{graphicx}
\usepackage{amssymb}
\usepackage{multirow}
\usepackage{booktabs}
\usepackage{amssymb}
\usepackage{mathrsfs}
\usepackage{longtable}
\usepackage{lscape}
\usepackage{tabularx}
\usepackage{epsfig}
\usepackage{colortbl}
\usepackage{amsthm}
\usepackage{amsmath}
\usepackage{mathtools}
\newdefinition{definition}{Definition}
\newdefinition{example}{Example}
\newdefinition{remark}{Remark}
\newtheorem{theorem}{Theorem}[section]

\newdefinition{property}{Property}

\biboptions{square,sort&compress}

\journal{arXiv.org}

\begin{document}

\begin{frontmatter}



\title{Basic concepts, definitions, and methods in D number theory}

\author[NWPU]{Xinyang Deng\corref{COR}}
\ead{xinyang.deng@nwpu.edu.cn}

\cortext[COR]{Corresponding author.}

\address[NWPU]{School of Electronics and Information, Northwestern Polytechnical University, Xi'an 710072, China}

\begin{abstract}
As a generalization of Dempster-Shafer theory, D number theory (DNT) aims to provide a framework to deal with uncertain information with non-exclusiveness and incompleteness. Although there are some advances on DNT in previous studies, however, they lack of systematicness, and many important issues have not yet been solved. In this paper, several crucial aspects in constructing a perfect and systematic framework of DNT are considered. At first the non-exclusiveness in DNT is formally defined and discussed. Secondly, a method to combine multiple D numbers is proposed by extending previous exclusive conflict redistribution (ECR) rule. Thirdly, a new pair of belief and plausibility measures for D numbers are defined and many desirable properties are satisfied by the proposed measures. Fourthly, the combination of information-incomplete D numbers is studied specially to show how to deal with the incompleteness of information in DNT. In this paper, we mainly give relative math definitions, properties, and theorems, concrete examples and applications will be considered in the future study.
\end{abstract}

\begin{keyword}
Non-exclusiveness \sep Combination rule \sep Belief measure \sep Plausibility measure \sep D number theory \sep Dempster-Shafer theory
\end{keyword}
\end{frontmatter}

\section{Introduction}
Dempster-Shafer theory (DST) \cite{Dempster1967,shafer1976mathematical}, also called evidence theory or belief function theory, is one of the most popular theories for dealing with uncertain information, and has been widely used in many fields. Although DST has many advantages in representing and dealing with uncertainty, but it is limited by some hypotheses and constraints that are hardly satisfied in some situation \cite{yen1990generalizing,smets1994transferable,dezert2002foundations,dezert2009introduction}. There are two main aspects. First, in DST a frame of discernment (FOD) must be composed of mutually exclusive elements, which is called the FOD's exclusiveness hypothesis. Second, in DST the sum of basic probabilities or belief $m(.)$ in a basic probability assignment (BPA) must be 1 (or basic probabilities can not be assigned to elements outside the FOD), which is called the BPA's completeness constraint.

To overcome the above-mentioned limitations in DST, a new generalization of DST, called D number theory (DNT), has been proposed in recently \cite{xydeng2017DNCR,deng2019d} for the fusion of uncertain information with non-exclusiveness and incompleteness. The theory of DNT stems from the concept of D numbers \cite{deng2012d,deng2014supplier411,XDengEIA2014,deng2015d,xiao2018novel,li2018dCC,seiti2019developing,mo2018new}, and aims to build a more sophisticated framework for representing and reasoning with uncertain information similar to DST from a generic set-membership perspective, in which DNT relaxes the exclusiveness constraint of elements in FOD and completeness assumption of BPA in DST.

DNT is a developing theory. In previous studies, the definition of D numbers, combination rule and uncertainty measure for D numbers, and others have been studied one after another \cite{xydeng2017DNCR,deng2017fuzzy,deng2019evaluating,IJISTUDNumbers,li2016novel}. Although there are some advances on DNT, however, they lack of systematicness, and many important issues have not yet been solved. Especially, these key aspects still remain unsolved: (1) How to formally describe and represent the non-exclusiveness; (2) How to effectively combine multiple D numbers; (3) Lack of a pair of desirable belief and plausibility measures for D numbers; (4) How to appropriately handle information-incomplete D numbers. These aspects are crucial components in establishing a perfect and systematic framework of DNT. In this paper, we center on these aspects, and propose many new concepts, definitions, and methods to improve the theoretical framework of DNT. Concretely, at first the non-exclusiveness in DNT is formally defined and discussed. Secondly, a method to combine multiple D numbers is proposed by extending previous exclusive conflict redistribution (ECR) rule \cite{xydeng2017DNCR,deng2019d}. Thirdly, a new pair of belief and plausibility measures for D numbers are defined and some basic properties are proved. Fourthly, the combination of information-incomplete D numbers is studied specially to show how to deal with the incompleteness of information in DNT. With the studies of this paper, DNT is much more close to a compatible generalization of DST.

The rest of this paper is organized as follows. Section \ref{SectBasicsDST} gives a brief introduction about DST. In Section \ref{SectionProposed}, several important aspects in the framework of DNT are studied. Finally, Section \ref{SectConclusion} concludes the paper.

\section{Basics of Dempster-Shafer theory}\label{SectBasicsDST}
In this section, some basic definitions and concepts about DST are given as below \cite{shafer1976mathematical}.

Let $\Omega$ be a set of $N$ mutually exclusive and collectively exhaustive events, indicated by
\begin{equation}
\Omega  = \{ q_1 ,q_2 , \cdots ,q_i , \cdots ,q_N \}
\end{equation}
where set $\Omega$ is called a frame of discernment (FOD). The power set of $\Omega$ is indicated by $2^\Omega$, namely
\begin{equation}
2^\Omega   = \{ \emptyset ,\{ q_1 \} , \cdots ,\{ q_N \} ,\{ q_1
,q_2 \} , \cdots ,\{ q_1 ,q_2 , \cdots ,q_i \} , \cdots ,\Omega \}.
\end{equation}
The elements of $2^\Omega$ or subsets of $\Omega$ are called propositions.

\begin{definition}
Let a FOD be $\Omega = \{ q_1 ,q_2 , \cdots, q_N \}$, a mass function, or basic probability assignment (BPA), defined on $\Omega$ is a mapping $m$ from  $2^\Omega$ to $[0,1]$, formally defined by:
\begin{equation}
m: \quad 2^\Omega \to [0,1]
\end{equation}
which satisfies the following condition:
\begin{eqnarray}
m(\emptyset ) = 0 \quad {\rm{and}} \quad \sum\limits_{A \subseteq \Omega }{m(A) = 1}
\end{eqnarray}
\end{definition}

\begin{definition}
Given a BPA, its associated belief measure $Bel_{m}$ and plausibility measure $Pl_{m}$ express the lower bound and upper bound of the support degree to each proposition in that BPA, respectively. They are defined as
\begin{equation}\label{EqBelMeasureDST}
Bel_{m} (A) = \sum\limits_{B \subseteq A} {m(B)}
\end{equation}
\begin{equation}\label{EqPlMeasureDST}
Pl_{m}(A) = \sum\limits_{B \cap A \ne \emptyset }{m(B)}
\end{equation}
\end{definition}

Obviously, $Pl_{m}(A) \ge Bel_{m}(A)$ for each $A \subseteq \Omega$, and $[Bel_{m}(A), Pl_{m}(A)]$ is called the belief interval of $A$ in $m$.

As a theory for uncertain information fusion, DST provides a basic combination rule called Dempster's rule to fuse multiple BPAs from independent information sources. This rule is formally defined as follows.

\begin{definition}
Assume there are two BPAs denoted as $m_1$ and $m_2$, let $m$ be the combination result by Dempster's rule, denoted as $m = m1 \oplus m_2$, then
\begin{equation}
m(A) = \left\{ {\begin{array}{*{20}l}
   {\frac{1}{{1 - K}}\sum\limits_{B \cap C = A} {m_1 (B)m_2 (C)} \;,} & {A \ne \emptyset ;}  \\
   {0\;,} & {A = \emptyset }.  \\
\end{array}} \right.
\end{equation}
with
\begin{equation}\label{K}
    K=\sum\limits_{B\bigcap C=\emptyset}m_1(B)m_{2}(C)
\end{equation}
\end{definition}

\section{D number theory as a compatible generalization of DST}\label{SectionProposed}
D number theory (DNT) is a new theoretical framework for uncertainty reasoning that has generalized DST to the situation of non-exclusive and incomplete information. In this section, the basic definition about D numbers proposed in previous studies is introduced firstly. Then, the non-exclusiveness in DNT is studied, a formal definition about the non-exclusiveness is presented, and relative theorems are exhibited. Thirdly, as for the combination of multiple D numbers, a new ECR rule is proposed to implement the combination of multiple D numbers. Next, a new pair of belief and plausibility measures for D numbers are proposed and their basic properties are proved. At last, the combination of information-incomplete D numbers is studied and discussed specially.

\subsection{Definition of a D number}
D number is a new model to represent uncertain information which relaxes the exclusiveness assumption and completeness constraint of DST \cite{deng2012d,deng2014supplier411,deng2019d}.

\begin{definition}\label{DefDNumbers}
Let $\Theta$ be a nonempty finite set $\Theta  = \{ \theta_1 ,\theta_2 , \cdots ,\theta_N \}$, a D number is a mapping formulated by
\begin{equation}
D: 2^{\Theta} \to [0,1]
\end{equation}
with
\begin{eqnarray}
\sum\limits_{B \subseteq \Theta } {D(B)} \le 1  \quad {\rm {and}} \quad
D(\emptyset ) = 0
\end{eqnarray}
where $\emptyset$ is the empty set and $B$ is a subset of $\Theta$.
\end{definition}

If $\sum\limits_{B \subseteq \Theta } {D(B) = 1}$, the corresponding D number is information-complete. Otherwise, if $\sum\limits_{B \subseteq \Theta } {D(B) < 1}$ the D number is information-incomplete. The information integrity in a D number can be simply expressed by a Q value:

\begin{equation}
Q(D) = \sum\limits_{B \subseteq \Theta } {D(B)}
\end{equation}

In the following of this paper, we will first study relative concepts and definitions of DNT in the case of complete information, and later specifically give a subsection to discuss the case of incomplete information. But the proposed concepts, definitions and methods for information-complete D numbers can be naturally and easily generalized to the case of information-incomplete D numbers.

\subsection{Non-exclusiveness in DNT}
In DNT, the elements in FOD $\Theta$ are not required to be mutually exclusive. Regarding the non-exclusiveness in DNT, a membership function is developed to measure the non-exclusive degrees between elements in $\Theta$ \cite{xydeng2017DNCR}. In this paper, an axiomatic definition about the non-exclusive degree is proposed as below.

\begin{definition}\label{DefNonExclusiveness}
Given a nonempty set $\Theta$, for any $B_i ,B_j  \in 2^{\Theta}$ the non-exclusive degree between $B_i$ and $B_j$ is characterized by a mapping
\begin{equation}
u :2^{\Theta}   \times 2^{\Theta}   \to [0,1]
\end{equation}
satisfying the following properties

(P1) Generalized compatibility: $u (B_i ,B_j ) = 1$ if $B_i  \cap B_j  \ne \emptyset$, $u (B_i ,B_j ) = p$ where $0 \le p \le 1$ if $B_i  \cap B_j  = \emptyset$.

(P2) Symmetry: $u (B_i ,B_j ) = u (B_j ,B_i )$.

(P3) Monotonicity: $u(A,B_i) \le u(A,B_j)$ if $B_i \subseteq B_j$.

(P4) Triangle inequality: $u(A,B_i) + u(A,B_j) \ge u(A,B_i \cup B_j)$ for any $A, B_i, B_j \subseteq \Theta$.

(P5) Zero element $\emptyset$: $u(A,\emptyset) = u(\emptyset, \emptyset) = 0$ for $A \subseteq \Theta$.
\end{definition}

In the above definition, the five properties constitute a smallest set of necessary requirements for a rational non-exclusive degree $u$. Generalized compatibility brings about an essential difference between DNT and DST. In DST, if $B_i  \cap B_j  = \emptyset$ then the two sets are definitely incompatible, namely they are mutually exclusive. By contrast, in DNT since $u (B_i ,B_j ) = p$ in which $p$ can be a non-zero value is allowed for $B_i  \cap B_j  = \emptyset$, the exclusiveness assumption of DST is relaxed entirely, moreover the degree of non-exclusiveness between $B_i$ and $B_j$ is quantified by $p$. In this sense, DST is a special case of DNT in which generalized compatibility is replaced by classical compatibility. The rest properties are easy understood and accepted. For instance, given a set $A$ its non-exclusive degree with another set $B$ is certainly increasing with the size of $B$, which is reflected by the monotonicity. In addition, triangle inequality is also required naturally because these relationships are definitely satisfied by $u$: $u(A,A) + u(A,\bar A) \ge u(A, \Theta )$ where the universe is $\Theta$, $u(A,B) + u(A,B) \ge u(A,B \cup B)$ for any $A,B \subseteq \Theta$.

According to above definition about non-exclusive degree $u$, all non-exclusive degrees $u (B_i ,B_j )$, $B_i ,B_j  \subseteq {\Theta}$ and $B_i ,B_j  \ne {\emptyset}$, make up a $({2^{|\Theta |}} - 1) \times ({2^{|\Theta |}} - 1)$ symmetrical matrix ${\bf{U}}$. With respect to the matrix of non-exclusive degrees ${\bf{U}}$ in which assume $|\Theta| \ge 2$, the following theorems are considerable.

\begin{theorem}\label{ThermUfullRank}
$Rank({\bf{U}}) = {2^{|\Theta |}} - 1$ if and only if $\forall A \subseteq \Theta $ and $A \ne \emptyset$ there is $u(A, \bar A) \ne 1$.
\end{theorem}
\begin{proof}
For a nonempty set $\Theta$, there are $n = (2^{|\Theta|} - 2) / 2$ pairs of complementary sets $(A, \bar A)$ with $A \ne \emptyset$. Assume $u(A, \bar A) = p_i$, where $i = 1,\cdots, n$ and $0 \le p_i < 1$, we can obtain
\[
|{\bf{U}}| = \prod\limits_{i = 1}^n {(1 - {p_i})^2}
\]

Since for any symmetrical matrix {\bf{M}} it is of full rank if and only if $|{\bf{M}}| \ne 0$, hence $Rank({\bf{U}}) = {2^{|\Theta |}} - 1$ if and only if $|{\bf{U}}| \ne 0$, i.e., $p_i \ne 1$, $i = 1,\cdots, n$, namely $u(A, \bar A) \ne 1$ for any nonempty $A \subseteq \Theta$.
\end{proof}

\begin{theorem}\label{ThermUnotfullRank}
$Rank({\bf{U}}) < {2^{|\Theta |}} - 1$ if and only if $\exists A \subseteq \Theta $ and $A \ne \emptyset$ there is $u(A, \bar A) = 1$.
\end{theorem}

Since for any symmetrical matrix {\bf{M}} it is not of full rank if and only if $|{\bf{M}}| = 0$, the above theorem \ref{ThermUnotfullRank} can be proved as similar as the proof of Theorem \ref{ThermUfullRank}.

\begin{theorem}\label{ThermUfuzzyR}
Let $R$ be a fuzzy relation on a nonempty $\Theta$ satisfying

(1) Symmetry: $R(\theta_i,\theta_j) = R(\theta_j, \theta_i)$

(2) Binary reflexivity: $R(\theta,\theta) = 1$

(3) Normalization: $R(\theta_i,\theta_j) \in [0,1]$

\noindent for any non-empty $\theta, \theta_i,\theta_j \in \Theta$, and assume $R(\theta, \emptyset) = R(\emptyset, \emptyset) = 0$. Then, $u$ is a function about non-exclusive degrees between subsets of $\Theta$, which satisfies all properties in Definition \ref{DefNonExclusiveness}, if
\begin{equation}
u({B_i},{B_j}) = \mathop {\max }\limits_{{\theta _i} \in {B_i},{\theta _j} \in {B_j}} \{ \; R({\theta _i},{\theta _j}) \; \} \end{equation}
where $B_i, B_j \subseteq \Theta$.
\end{theorem}

Theorem \ref{ThermUfuzzyR} is easily proved and provides a simple method to construct the matrix of non-exclusive degrees based on a fuzzy relation about the non-exclusive degree between elements of $\Theta$. Moreover, if the underlying fuzzy relation $R$ meets $R(\theta_i,\theta_j) \ne 1$ if $\theta_i \ne \theta_j$, then the constructed matrix ${\bf{U}}$ is of full rank.

\subsection{Combination rules in DNT}
DST is often used to combining multiple information from independent sources, as an effective theory of uncertain information fusion. As a generalization of DST, DNT is also required to have the ability to fuse multiple information expressed by D numbers. In our previous studies \cite{xydeng2017DNCR,deng2019d}, an exclusive conflict redistribution (ECR) rule was proposed to combine two D numbers, which is defined as below.

\begin{definition}\label{DefECRruleD1D2}
Let $D_1$ and $D_2$ are two D numbers defined on $\Theta$, the combination of $D_1$ and $D_2$, denoted by $D = D_1 \odot D_2$, is defined as
\begin{equation}\label{EqDefECRD1D2}
D(A) = \left\{ \begin{array}{l}
 0,\quad A = \emptyset  \\
 \frac{1}{{1 - K_D }}\left( \begin{array}{l}
 {\sum\limits_{B \cap C = A} {D_1 (B)D_2 (C)}} \;\; +  \\
 \sum\limits_{\scriptstyle B \cup C = A \hfill \atop
  \scriptstyle B \cap C = \emptyset  \hfill} {u(B,C)D_1 (B)D_2 (C)}  \\
 \end{array} \right),\quad A \ne \emptyset  \\
 \end{array} \right.
\end{equation}
with
\begin{equation}\label{EqDefECRD1D2KD}
K_D  = \sum\limits_{B \cap C = \emptyset } {\left( {1 - u(B,C)} \right)D_1 (B)D_2 (C)}
\end{equation}
where $A, B, C \subseteq {\Theta}$.
\end{definition}

From the definition of above ECR rule, we can find that the ECR rule is an extension of Dempster's rule in DST. The idea behind ECR rule is that the quantity $D_1 (B)D_2 (C)$ can not be totally treated as conflict if $B \cap C = \emptyset$, because of the property of generalized compatibility in Definition \ref{DefNonExclusiveness}. Therefore, in the ECR rule, with regard to the dispose of $D_1 (B)D_2 (C)$, a belief of ${u(B,C)D_1 (B)D_2 (C)}$ is assigned to $B \cup C$, and the rest ${\left( {1 - u(B,C)} \right)D_1 (B)D_2 (C)}$ is identified as exclusive conflict. It is easily proved that the ECR rule is a compatible extension of Dempster's rule, and it can be totally reduced to classical Dempster's rule in some special cases.

By analyzing the ECR rule, we find that this rule is commutative, i.e., ${D_1} \odot {D_2} = {D_2} \odot {D_1}$, but it does not satisfies the associativity, namely $({D_1} \odot {D_2}) \odot {D_3} \ne {D_1} \odot ({D_2} \odot {D_3})$, which may lead to some difficulties in combining multi-source information by using this rule. In order to solve the problem, in this paper an extended ECR rule is proposed to combine multiple D numbers, which essentially still complies with the original idea of ECR rule. The new ECR rule for multiple D numbers is defined as below.

\begin{definition}\label{DefECRruleD1D2D3DN}
Let $D_1, D_2, \cdots, D_n$ be $n$ D numbers defined on $\Theta$, the combination of them, denoted by $D = D_1 \odot D_2 \odot \cdots \odot D_n$, is defined as
\begin{equation}\label{EqDefECRD1D2D3DN}
D(A) = \left\{ \begin{array}{l}
 0,\quad A = \emptyset  \\
 \frac{1}{{1 - {K_D}}}\left( \begin{array}{l}
 \sum\limits_{ \cap {B_j} = A} {\prod\limits_{1 \le i \le n} {{D_i}({B_j})} }  +  \\
 \sum\limits_{\scriptstyle  \cup {B_j} = A \hfill \atop
  \scriptstyle  \cap {B_j} = \emptyset  \hfill} {\mathop {\min }\limits_{{B_k},{B_l} \in \{ {B_j}\} } \left\{ {u({B_k},{B_l})} \right\}\prod\limits_{1 \le i \le n} {{D_i}({B_j})} }  \\
 \end{array} \right),\quad A \ne \emptyset  \\
 \end{array} \right.
\end{equation}\label{}
with
\begin{equation}\label{EqDefECRD1D2D3DNKD}
{K_D} = \sum\limits_{ \cap {B_j} = \emptyset } {\left( {1 - \mathop {\min }\limits_{{B_k},{B_l} \in \{ {B_j}\} } \left\{ {u({B_k},{B_l})} \right\}} \right)\prod\limits_{1 \le i \le n} {{D_i}({B_j})} }
\end{equation}
where $A, B_j, B_k, B_l \subseteq {\Theta}$.
\end{definition}

Obviously, the new ECR rule can be reduced to the original ECR rule in Definition \ref{DefECRruleD1D2} if $n = 2$. It is noted that in the new ECR rule all D numbers are combined simultaneously to implement the fusion of multiple uncertain information modeled by D numbers. Compared with our previous solutions on combining multiple D numbers, for example induced ordering combination method \cite{deng2017fuzzy} and weighted average combination method \cite{deng2019d}, the new solution is more reasonable and natural.

\subsection{Belief and plausibility measures for D numbers}
Belief and plausibility measures are two equivalent forms of BPA in DST, which physically express the lower bound and upper bound of the support degree to each proposition $A$ in that BPA, respectively. In our previous studies, a belief measure and a plausibility measure for D numbers were developed \cite{deng2019d,IJISTUDNumbers}. However, after deep and further research, it was found that the previous developed belief and plausibility measures of D numbers are not satisfactory. For instance, they do not satisfy the property of $Bel(A) + Pl(\bar A) = 1$ for any subset $A$ belonging to the FOD. Facing that, in this paper we present a new pair of belief and the plausibility measures, and the properties of the new belief and plausibility measures are discussed.

\begin{definition}\label{DefBelDNT}
Let $D$ represent a D number defined on $\Theta$, the belief measure for $D$ is mapping
\[
Bel: 2^{\Theta}   \to [0,1]
\]
satisfying
\begin{equation}\label{EqBelMeasureDNT}
Bel(A) = \sum\limits_{B \subseteq A} {D(B)\left[ {1 - u(B,\bar A)} \right]}
\end{equation}
for any $A \subseteq \Theta$. Since ${u(B,\bar A)} = 1$ for any $B \not\subset A$ in terms of the definition of non-exclusive degree $u$, Eq. (\ref{EqBelMeasureDNT}) can be written as
\begin{equation}\label{EqBelMeasure}
Bel(A) = \sum\limits_{B \subseteq \Theta} {D(B)\left[ {1 - u(B,\bar A)} \right]}
\end{equation}
\end{definition}

\begin{definition}
Let $D$ represent a D number defined on $\Theta$, the plausibility measure for $D$ is mapping
\[
Pl: 2^{\Theta}  \to [0,1]
\]
satisfying
\begin{equation}
Pl(A) = \sum\limits_{B \cap A \ne \emptyset } {D(B)}  + \sum\limits_{B \cap A = \emptyset } {u(B,A)D(B)}
\end{equation}
where $A, B \subseteq \Theta$. Because $u(B ,A) = 1$ for $B  \cap A  \ne \emptyset$ in terms of the definition of non-exclusive degree $u$, the plausibility measure $Pl$ can also be written as
\begin{equation}\label{EqPlMeasureDNT}
Pl(A) = \sum\limits_{B \subseteq \Theta} {u(B,A) D(B)}
\end{equation}
\end{definition}

If letting ${\bf{Pl}}$ be the vector form of plausibility measure, and {\bf{D}} the vector form of a D number, according to Eq. (\ref{EqPlMeasureDNT}), we have the following relationship
\begin{equation}
{\bf{Pl}} = {\bf{D}} \cdot {\bf{U}}
\end{equation}

In terms of the above definitions, a D number ${\bf{D}}$, and its belief measure ${\bf{Bel}}$ and plausibility measure ${\bf{Pl}}$ are easily obtained. Moreover, if ${\bf{U}}$ is of full rank, {\bf{D}}, ${\bf{Bel}}$ and ${\bf{Pl}}$ have one-to-one correspondence to others. As similar as DST, $[Bel(A), Pl(A)]$ forms a belief interval about the possibility of $A$ in DNT. It is easy to find that the $Bel$ and $Pl$ for D numbers will degenerate to classical belief measure and plausibility measure in DST if the associated D number is a BPA in fact.

Some desirable properties are satisfied by the proposed measures.

\begin{theorem}\label{ThermBelPl}
\[
Bel(A) \le Pl(A)
\]
\end{theorem}
\begin{proof}
\[
\begin{array}{l}
 Pl(A) - Bel(A) = \sum\limits_{B \subseteq \Theta } {D(B)u(B,A)}  - \sum\limits_{B \subseteq \Theta } {D(B)[1 - u(B,\bar A)]}  \\
 \quad \quad \quad \quad \quad \quad \quad  = \sum\limits_{B \subseteq \Theta } {D(B)[u(B,A) + u(B,\bar A) - 1]}  \\
 \end{array}
 \]

According to the definition of non-exclusive degree $u$, for any nonempty set $B$ either $u(B,A)$ or $u(B,\bar A)$ must be 1, therefore $u(B,A) + u(B,\bar A) - 1 \ge 1$ for any $B \ne \emptyset$. Thus
\[
Pl(A) \ge Bel(A)
\]
\end{proof}

In terms of the above proof of Theorem \ref{ThermBelPl}, we can further obtain that
\[
Pl(A) - Bel(A) = \sum\limits_{B \subseteq A} {D(B)u(B,\bar A)}  + \sum\limits_{B \subseteq \bar A} {D(B)u(B,A)}  + \sum\limits_{\scriptstyle B \cap A \ne \emptyset  \hfill \atop
  \scriptstyle B \cap \bar A \ne \emptyset  \hfill} {D(B)}
\]
which provides a simple way to calculate the length of interval $[Bel(A), Pl(A)]$ that is often used to represent the
degree of imprecision of $A$.

\begin{theorem}
\[Bel(A) + Pl(\bar A) = 1\]
\end{theorem}
\begin{proof}
\[
\begin{array}{l}
 Bel(A) + Pl(\bar A) = \sum\limits_{B \subseteq \Theta } {D(B)[1 - u(B, \bar A)]}  + \sum\limits_{B \subseteq \Theta } {D(B)u(B, \bar A)}  \\
 \quad \quad \quad \quad \quad \quad \quad  = \sum\limits_{B \subseteq \Theta } {D(B)[1 - u(B, \bar A) + u(B, \bar A)]}  \\
 \quad \quad \quad \quad \quad \quad \quad  = \sum\limits_{B \subseteq \Theta } {D(B)}  = 1 \\
 \end{array}
 \]
\end{proof}

\begin{theorem}
\[
Bel(A) + Bel(\bar A) \le 1, \quad Pl(A) + Pl(\bar A) \ge 1
\]
\end{theorem}
\begin{proof}
Since $Bel(A) \le Pl(A)$ and $Bel(\bar A) + Pl( A) = 1$, we have $Bel(A) \le 1 - Bel(\bar A)$, hence
\[
Bel(A) + Bel(\bar A) \le 1
\]

Similarly, $Pl(A) + Pl(\bar A) \ge 1$ can be proved.
\end{proof}

\begin{theorem}
If $A \subseteq B$, then $Bel(A) \le Bel(B)$ and $Pl(A) \le Pl(B)$
\end{theorem}
\begin{proof}
\[
\begin{array}{l}
 Pl(B) - Pl(A) = \sum\limits_{C \subseteq \Theta } {D(C)u(C,B)}  - \sum\limits_{C \subseteq \Theta } {D(C)u(C,A)}  \\
 \quad \quad \quad \quad \quad \quad  = \sum\limits_{C \subseteq \Theta } {D(C)[u(C,B) - u(C,A)]}  \\
 \end{array}
\]
Because $u(C,A) \le u(C,B)$ for $A \subseteq B$, then
\[
Pl(B) \ge Pl(A)
\]

Similarly, $Bel(B) \ge Bel(A)$ can be proved for $A \subseteq B$.
\end{proof}

In DST, given a mass function $m$, its corresponding belief measure $Bel _{m}$ can be obtained by means of Eq.(\ref{EqBelMeasureDST}), Shafer \cite{shafer1976mathematical} has proved that the belief measure $Bel _{m}$ is a belief function. As a generalization of DST, DNT also retain the characteristic.

\begin{definition}
Given FOD $\Theta$, a function $bel: 2^{\Theta}   \to [0,1]$ is a belief function if and only if

(1) $bel(\emptyset) = 0$;

(2) $bel(\Theta) = 1$;

(3) For all $A_1, \cdots, A_n \subseteq \Theta$,
\[bel({A_1} \cup  \cdots  \cup {A_n}) \ge \sum\limits_{\scriptstyle I \subseteq \{ 1, \cdots ,n\}  \hfill \atop
  \scriptstyle I \ne \emptyset  \hfill} {{{( - 1)}^{|I| + 1}}bel\left( {\bigcap\limits_{i \in I} {{A_i}} } \right)} \]
\end{definition}

\begin{theorem}\label{ThermBelBeliefFunction}
Belief measure $Bel$ in Definition \ref{DefBelDNT} is a belief function.
\end{theorem}
\begin{proof}
It is obvious that
(1) $Bel(\emptyset ) = 0$ and (2) $Bel(\Theta) = 1$. Now let us prove the condition (3) of belief function.

In Shafer's book \cite{shafer1976mathematical}, there is a lemma:

If $A$ is a finite set, then
$\sum\limits_{B \subseteq A} {{{( - 1)}^{|B|}}}  = \left\{ \begin{array}{l}
 1\quad {\rm{if}}\;A = \emptyset  \\
 0\quad {\rm{otherwise.}} \\
 \end{array} \right.$

For all $A_1, \cdots, A_n \subseteq \Theta$, let $I(B) = \{ i|1 \le i \le n;\;B \subseteq {A_i}\}$. According to the above lemma, we have
\[\begin{array}{l}
 \sum\limits_{\scriptstyle I \subseteq \{ 1, \cdots ,n\}  \hfill \atop
  \scriptstyle I \ne \emptyset  \hfill} {{{( - 1)}^{|I| + 1}}Bel\left( {\bigcap\limits_{i \in I} {{A_i}} } \right)}  \\
 \quad \quad \quad \quad \quad \quad  = \sum\limits_{\scriptstyle I \subseteq \{ 1, \cdots ,n\}  \hfill \atop
  \scriptstyle I \ne \emptyset  \hfill} {{{( - 1)}^{|I| + 1}}\sum\limits_{B \subseteq \bigcap\limits_{i \in I} {{A_i}} } {D(B)[1 - u(B,\overline{\bigcap\limits_{i \in I} {{A_i}}} )]} }  \\
 \quad \quad \quad \quad \quad \quad  = \sum\limits_{\scriptstyle B \subseteq \Theta  \hfill \atop
  \scriptstyle I(B) \ne \emptyset  \hfill} {D(B)[1 - u(B,\overline{\bigcap\limits_{i \in I(B)} {{A_i}}} )]\sum\limits_{\scriptstyle I \subseteq I(B) \hfill \atop
  \scriptstyle I \ne \emptyset  \hfill} {{{( - 1)}^{|I| + 1}}} }  \\
 \quad \quad \quad \quad \quad \quad  = \sum\limits_{\scriptstyle B \subseteq \Theta  \hfill \atop
  \scriptstyle I(B) \ne \emptyset  \hfill} {D(B)[1 - u(B,\overline{\bigcap\limits_{i \in I(B)} { {A_i}}} )]\left( {1 - \sum\limits_{I \subseteq I(B)} {{{( - 1)}^{|I|}}} } \right)}  \\
 \quad \quad \quad \quad \quad \quad  = \sum\limits_{\scriptstyle B \subseteq \Theta  \hfill \atop
  \scriptstyle I(B) \ne \emptyset  \hfill} {D(B)[1 - u(B,\overline{\bigcap\limits_{i \in I(B)} { {A_i}}} )]}  \\
 \quad \quad \quad \quad \quad \quad  \le \sum\limits_{\scriptstyle B \subseteq \Theta  \hfill \atop
  \scriptstyle I(B) \ne \emptyset  \hfill} {D(B)[1 - u(B,\overline{\bigcup\limits_{i \in I(B)} { {A_i}}} )]}  \\
 \quad \quad \quad \quad \quad \quad  \le \sum\limits_{\scriptstyle B \subseteq \Theta  \hfill \atop
  \scriptstyle I = \{ 1, \cdots ,n\}  \hfill} {D(B)[1 - u(B,\overline{\bigcup\limits_{i \in I} {  {A_i}}} )]}  \\
 \quad \quad \quad \quad \quad \quad {\rm{ = }}Bel({A_1} \cup  \cdots  \cup {A_n}) \\
 \end{array}\]

Namely
\[Bel({A_1} \cup  \cdots  \cup {A_n}) \ge \sum\limits_{\scriptstyle I \subseteq \{ 1, \cdots ,n\}  \hfill \atop
  \scriptstyle I \ne \emptyset  \hfill} {{{( - 1)}^{|I| + 1}}Bel\left( {\bigcap\limits_{i \in I} {{A_i}} } \right)} \]

Therefore, the belief measure $Bel$, where $Bel(A) = \sum\limits_{B \subseteq A} {D(B)\left[ {1 - u(B,\bar A)} \right]}$ for any $A \subseteq \Theta$, is a belief function.
\end{proof}

DST is also called belief function theory because its belief measure $Bel_{m}$ is a belief function and many other concepts can be derived from $Bel_{m}$. Theorem \ref{ThermBelBeliefFunction} shows that the belief measure $Bel$ of D numbers is also a belief function, therefore we have more reasons to say that DNT is a compatible generalization of DST.

\subsection{Combination of incomplete information in DNT}
In above subsections, the D numbers are assumed to be information-complete. In this subsection, we will study the case of information-complete D numbers, and use the combination of incomplete information in DNT as an example to show how to handle the information-incomplete D numbers.

In our previous work on DNT \cite{deng2019d,IJISTUDNumbers}, for the information-incomplete case, an unknown event $X$ is introduced to enlarge the FOD from $\Theta$ to $\Theta  \cup \{X\}$, and to transform a D number with incomplete information represented by $D$ to information-complete D number $D_t$ such that
\begin{equation}
{\sum\limits_{B \subseteq {\Theta \cup \{X\}} } {D_t(B)} } = 1  \quad {\rm {and}} \quad
D_t(\emptyset ) = 0
\end{equation}
Note that there is not any restriction for the relationship between new introduced $X$ and original $\Theta$. Then, based on the new obtained information-complete D number, the ECR rule can be used on a new universe ${\Theta \cup \{X\}}$.

However, by further study, we find that there is a big difficulty to execute the above process: how to obtain $u (X ,A )$ required by the ECR rule, where $A \subseteq \Theta$, since there is not any restriction on $X$. For a totally unknown object $X$, any supposition on $u(X ,A )$ is controversial. In order to solve the puzzle, in this paper a new method for the combination of incomplete information in DNT is proposed. Within the proposed method, the concept of unknown event $X$ is clarified and redefined, and a new parameter $\delta$ is imported. Regarding $X$ and $\delta$, we have the following definition.

\begin{definition}
Given a FOD $\Theta$, $X$ is the set of any possible elements outside of $\Theta$, namely
\begin{equation}
X = \{ x|x \cap \Theta  = \emptyset ,x \ne \emptyset \}
\end{equation}
Moreover, $X$ and $\Theta$ are assumed to be completely exclusive, namely
\begin{equation}
u (S,F ) = 0,\;\forall S \subseteq X,F \subseteq \Theta
\end{equation}
\end{definition}

\begin{definition}
Given a FOD $\Theta$, $\delta$ is defined to represent the completeness degree of $\Theta$, and $\delta \in [0,1]$. If $\Theta$ is 100\% complete or exhaustive, then $\delta = 1$; Otherwise, $0 \le \delta < 1$, and the lower the completeness of $\Theta$ the smaller the value of $\delta$.
\end{definition}

From the above definitions, $X$ becomes a very clear conception, by contrast $\delta$ is a parameter needing to be estimated in concrete circumstances. Based on the defined $X$ and $\delta$, a new method to combine D numbers, especially two information-incomplete D numbers, denoted as $D_1$ and $D_2$, is proposed. If we need to combine multiple D numbers, the above proposed new version of ECR rule for $n$ D numbers can be directly used.

\begin{itemize}
  \item \textbf{(Step 1.)} In terms of new defined $X$ and $\delta$, transform the D number with incomplete information to an information-complete D number by operation $D_t = T_{new}(D)$
      \begin{equation}\label{DDTransformation}
      \left\{ {\begin{array}{*{20}{l}}
   {{D_t}(B) = D(B),\quad B \subset \Theta }  \\
   {{D_t}(\Theta) = D(\Theta ) + \delta \left[ {1 - Q(D)} \right]}  \\
   {{D_t}(X) = (1 - \delta )\left[ {1 - Q(D)} \right]}  \\
\end{array}} \right.
      \end{equation}
so that there is also
\begin{equation}
{\sum\limits_{B \subseteq {\Theta \cup \{X\}} } {D_t(B)} } = 1  \quad {\rm {and}} \quad
D_t(\emptyset ) = 0
\end{equation}

  \item \textbf{(Step 2.)} Suppose $D_1 = T_{new}(D_1)$ and $D_2= T_{new}(D_2)$, then their combination $D = D_1 \odot  D_2$ is obtained by using ECR rule's equations (\ref{EqDefECRD1D2}) and (\ref{EqDefECRD1D2KD}).
\end{itemize}

The properties of proposed combination method for information-incomplete D numbers are discussed as follows.

Suppose there are two D numbers $D_1$ and $D_2$ on FOD $\Theta$, since they may be information-incomplete, therefore $\sum\limits_{B \subseteq \Theta } {{D_1}(B)}  \le 1$, $\sum\limits_{C \subseteq \Theta } {{D_2}(C)}  \le 1$, and ${D_1}(\emptyset ) = {D_2}(\emptyset ) = 0$. The Q values of these two D numbers are $Q_1 = \sum\limits_{B \subseteq \Theta } {{D_1}(B)}$ and $Q_2 = \sum\limits_{C \subseteq \Theta } {{D_2}(C)}$, respectively. In terms of those information, a quantity denoted as $K_D^1$ can be calculated which represents the definitely known conflict between $D_1$ and $D_2$:
\[
K_D^1 = \sum\limits_{B \cap C = \emptyset } {(1 - {u}(B,C)){D_1}(B){D_2}(C)}
\]
where $B,C \subseteq \Theta$. Obviously, $K_D^1 \in [0,Q_1 Q_2]$.

Based on the proposed method, for an information-incomplete D number, it is firstly transformed to the information-complete situation through Eq. (\ref{DDTransformation}). Assume the completeness degree of FOD $\Theta$ is $\delta$. Hence, for $D_1$, the missing belief $(1-Q_1)$ is assigned as follows
\[\left\{ \begin{array}{l}
 {D_1}(\Theta ) = (1 - {Q_1})\delta  \\
 {D_1}(X) = (1 - {Q_1})(1 - \delta ) \\
 \end{array} \right.\]
Similarly, for $D_2$ we have
\[\left\{ \begin{array}{l}
 {D_2}(\Theta ) = (1 - {Q_2})\delta  \\
 {D_2}(X) = (1 - {Q_2})(1 - \delta ) \\
 \end{array} \right.\]

Since there are new beliefs assigned $\Theta$ and $X$ in $D_1$ and $D_2$, it leads to new conflict between the two D numbers. According to the definition of $X$, there is ${u}(S,F ) = 0,\;\forall S \subseteq X,F \subseteq \Theta$, hence the conflict causing by new assigned beliefs in D numbers can be obtained
\[
K_D^2 = {Q_2}(1 - {Q_1})(1 - \delta) + {Q_1}(1 - {Q_2})(1 - \delta) + 2(1 - {Q_1})(1 - {Q_2})\delta(1 - \delta)
\]
which can also be written as
\[K_D^2 = {D_1}(X) + {D_2}(X) - 2 {D_1}(X){D_2}(X)\]

All conflict between $D_1$ and $D_2$ are entirely counted in $K_D^1$ and $K_D^2$, hence the conflict coefficient $K_D$ in Eq. (\ref{EqDefECRD1D2KD}) can be expressed as
\[
{K_D} = K_D^1 + K_D^2
\]
Naturally, $1 - K_D^1 - K_D^2 \ge 0$. Then, by using the ECR rule, the result of combining $D_1$ and $D_2$ (after transformation), denoted as $D$, can be obtained. In $D$, we will have
\begin{equation}\label{EqDX}
 \centering
D(X) = \frac{{D_{1}(X)D_{2}(X)}}{{1 - {K_D}}} = \frac{{(1 - {Q_1})(1 - {Q_2}){{(1 - \delta)}^2}}}{{1 - K_D^1 - K_D^2}}
\end{equation}
where $D_1(X) = (1-\delta)(1-Q_1)$ and $D_2(X) = (1-\delta)(1-Q_2)$. For information-incomplete D numbers, in the combination result the final $D(X)$ is the concerned in this paper. Some properties related to $D(X)$ in the combination result are given as follows.

\begin{property}
$D(X)\propto K_D^1$ and $D(X)\propto K_D^2$.
\end{property}

This property is evident according to Eq. (\ref{EqDX}).

\begin{property}
$D(X)\propto 1/\delta$
\end{property}

\begin{proof}
In terms of Eq. (\ref{EqDX}), by calculating the partial derivative of $D(X)$ with respect to $\delta$, we have

\[
\frac{{\partial D(X)}}{{\partial \delta }} = \frac{{(1 - {Q_1})(1 - {Q_2})(1 - \delta )(\delta {Q_1} + \delta {Q_2} - {Q_1} - {Q_2} - 2\delta  + 2K_D^1)}}{{{{(1 - K_D^1 - K_D^2)}^2}}}
\]

Let $L = \delta {Q_1} + \delta {Q_2} - {Q_1} - {Q_2} - 2\delta  + 2K_D^1$. Since $K_D^1 \in [0,Q_1 Q_2]$, therefore $L$ gets the maximum value when $K_D^1 = Q_1 Q_2$, namely
\[\begin{array}{l}
 \max \;L = \delta {Q_1} + \delta {Q_2} - {Q_1} - {Q_2} - 2\delta  + 2{Q_1}{Q_2} \\
 \quad \quad \;\;\; = (\delta  + {Q_1})({Q_2} - 1) + (\delta  + {Q_2})({Q_1} - 1) \\
 \quad \quad \;\;\; \le 0 \\
 \end{array}\]
where $Q_1, Q_2, \delta \in [0,1]$. Thus, $\frac{{\partial D(X)}}{{\partial \delta }} \le 0$ holds, which implies that $D(X)$ is negatively correlated with $\delta$. Especially, $D(X) = 0$ while $\delta = 1$, and $D(X) = \frac{{(1 - {Q_1})(1 - {Q_2})}}{{1 - K_D^1 - {Q_1}(1 - {Q_2}) - {Q_2}(1 - {Q_1})}}$ while $\delta = 0$.
\end{proof}

\begin{property}
$D(X)\propto 1/Q_{1}$ and $D(X)\propto 1/Q_{2}$
\end{property}

\begin{proof}
By calculating the partial derivative of $D(X)$ with respect to $Q_1$, we have
\[
\frac{{\partial D(X)}}{{\partial Q_1 }} = \frac{{(1 - {Q_2}){{(1 - \delta )}^2}(\delta {Q_2} - {Q_2} - \delta  + K_D^1)}}{{{{(1 - K_D^1 - K_D^2)}^2}}}
\]

Since $K_D^1 \in [0,Q_1 Q_2]$ and $Q_1, Q_2, \delta \in [0,1]$, it can be easily obtained
\[
\frac{{\partial D(X)}}{{\partial Q_1}} \le 0
\]

Thus, $D(X)$ is negatively correlated with $Q_1$. Especially, $D(X) = 0$ if $Q_1 = 1$, $D(X) = \frac{{(1 - {Q_2}){{(1 - \delta )}^2}}}{{1 - K_D^1 - {Q_2}(1 - \delta ) - 2\delta (1 - {Q_2})(1 - \delta )}}$ if $Q_1 = 0$.

Similarly, $D(X)\propto 1/Q_{2}$ can also be proved.
\end{proof}

\begin{property}\label{PropertyDXD1D2}
Let $1 - K_D \ne 0$, then

(1) $D (X) = 0$, if and only if $D_1(X) = 0$ or $D_2(X) = 0$;

(2) $D (X) = D_{1}(X)$ if $1-K_D = D_{2}(X)$;

(3) $D (X) = D_{2}(X)$ if $1-K_D = D_{1}(X)$;

(4) $D (X) > D_{1}(X)$, if $1-K_D < D_{2}(X)$ and $D_1(X) \ne 0$;

(5) $0 < D (X) < D_{1}(X) $, if $1-K_D > D_{2}(X) > 0$ and $D_1(X) \ne 0$;

(6) $D (X) > D_{2}(X)$, if $1-K_D < D_{1}(X)$ and $D_2(X) \ne 0$;

(7) $0 < D (X) < D_{2}(X) $, if $1-K_D > D_{1}(X) > 0$ and $D_2(X) \ne 0$;

(8) $D (X) = 1$, if $1-K_D = 1$ and $D_1(X) \ne 0$ and $D_2(X) \ne 0$.
\end{property}

According to $K_D = K_D^1 + K_D^2$, $K_D^2 = {D_1}(X) + {D_2}(X) - 2 {D_1}(X){D_2}(X)$, and $D(X) = \frac{{D_{1}(X)D_{2}(X)}}{{1 - {K_D}}}$, this property can be proved easily. In terms of the property, the distribution of $D(X)$ respecting to $1 - K_D$ is obtained, as shown in Figure \ref{FigDXD1D2} in which supposing $0 < D_1(X) \le D_2(X)$.

 \begin{figure}[htbp]
   \centering
    \includegraphics[scale=0.65]{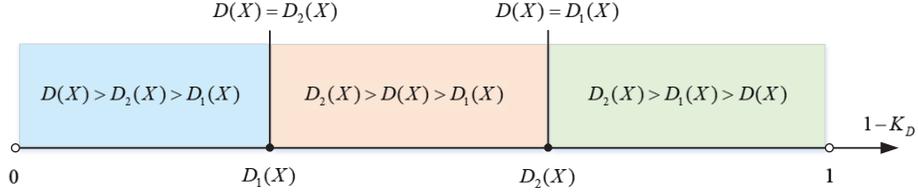}
    \caption{The distribution of $D(X)$ with to $1 - K_D$ separated by $D_1(X)$ and $D_2(X)$}
    \label{FigDXD1D2}
 \end{figure}

Moreover, based on Property \ref{PropertyDXD1D2}, it is easy to derive another property as follows where the relationship between $D(X)$ and $(1-Q_1)$, $(1- Q_2)$ is displayed.

\begin{property}
Let $0 < 1 - K_D < 1$, $D_1(X) \ne 0$ and $D_2(X) \ne 0$, then

(1) $D (X) = (1-Q_1)$ if $1-K_D = (1-\delta) D_{2}(X)$;

(2) $D (X) = (1-Q_2)$ if $1-K_D = (1-\delta) D_{1}(X)$;

(3) $D (X) > (1 - Q_1)$ if $1-K_D < (1-\delta) D_{2}(X)$;

(4) $D (X) > (1 - Q_2)$ if $1-K_D < (1-\delta) D_{1}(X)$;

(5) $D (X) < (1 - Q_1)$ if $1-K_D > (1-\delta) D_{2}(X)$;

(6) $D (X) < (1 - Q_2)$ if $1-K_D > (1-\delta) D_{1}(X)$.
\end{property}

As same as above, Figure \ref{FigDXQ1Q2} gives the distribution of $D(X)$ with respect to $1-K_D$, but shows the relationship between $D(X)$ and $(1-Q_1)$, $(1- Q_2)$, in which assume $0 < D_1(X) \le D_2(X)$.

 \begin{figure}[htbp]
   \centering
    \includegraphics[scale=0.65]{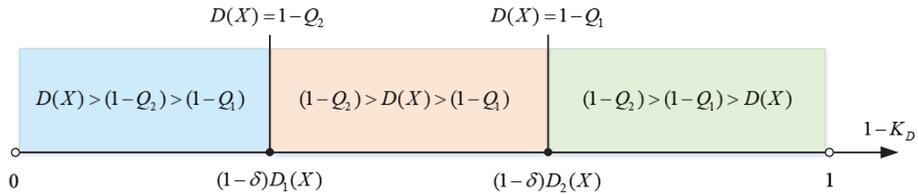}
    \caption{The distribution of $D(X)$ with to $1 - K_D$ separated by $1 - Q_1$ and $1 - Q_2$}
    \label{FigDXQ1Q2}
 \end{figure}

Having the above several properties, the characteristic of the proposed method for combining information-incomplete D numbers is basically clear, the imported $X$ to represent unknown is not completely unknown to us now.

\section{Conclusion}\label{SectConclusion}
In the paper, some basic concepts, definitions, and methods, including the exclusiveness, combination rules, belief and plausibility measures, combination of incomplete information, in the theoretical framework of DNT are studied. These issues are very important in establishing a perfect and systematic DNT. The research in this paper strengthens the mathematical foundation of DNT. In the future study, we will further enrich the theoretical and practical research of DNT.

%
%\section*{Acknowledgments}
%The work is partially supported by National Natural Science Foundation of China (Program Nos. 61703338, 61671384).



%% References


\bibliographystyle{elsarticle-num}
\bibliography{references}

\begin{thebibliography}{10}
\expandafter\ifx\csname url\endcsname\relax
  \def\url#1{\texttt{#1}}\fi
\expandafter\ifx\csname urlprefix\endcsname\relax\def\urlprefix{URL }\fi
\expandafter\ifx\csname href\endcsname\relax
  \def\href#1#2{#2} \def\path#1{#1}\fi

\bibitem{Dempster1967}
A.~P. Dempster, Upper and lower probabilities induced by a multivalued mapping,
  Annals of Mathematics and Statistics 38~(2) (1967) 325--339.

\bibitem{shafer1976mathematical}
G.~Shafer, A Mathematical Theory of Evidence, Princeton University Press,
  Princeton, 1976.

\bibitem{yen1990generalizing}
J.~Yen, Generalizing the {D}empster-{S}chafer theory to fuzzy sets, IEEE
  Transactions on Systems, man, and Cybernetics 20~(3) (1990) 559--570.

\bibitem{smets1994transferable}
P.~Smets, R.~Kennes, The transferable belief model, Artificial intelligence
  66~(2) (1994) 191--234.

\bibitem{dezert2002foundations}
J.~Dezert, Foundations for a new theory of plausible and paradoxical reasoning,
  Information and Security 9 (2002) 13--57.

\bibitem{dezert2009introduction}
J.~Dezert, F.~Smarandache, An introduction to DSmT, Infinite Study, 2009.

\bibitem{xydeng2017DNCR}
X.~Deng, W.~Jiang, Exploring the combination rules of {D} numbers from a
  perspective of conflict redistribution, in: Proceedings of the 20th
  International Conference on Information Fusion (FUSION), Xi'an, China, 2017,
  pp. 542--547.
\newblock \href {http://dx.doi.org/10.23919/ICIF.2017.8009696}
  {\path{doi:10.23919/ICIF.2017.8009696}}.

\bibitem{deng2019d}
X.~Deng, W.~Jiang, D number theory based game-theoretic framework in
  adversarial decision making under a fuzzy environment, International Journal
  of Approximate Reasoning 106 (2019) 194--213.

\bibitem{deng2012d}
Y.~Deng, D numbers: theory and applications, Journal of Information \&
  Computational Science 9~(9) (2012) 2421--2428.

\bibitem{deng2014supplier411}
X.~Deng, Y.~Hu, Y.~Deng, S.~Mahadevan, Supplier selection using {AHP}
  methodology extended by {D} numbers, Expert Systems with Applications 41~(1)
  (2014) 156--167.

\bibitem{XDengEIA2014}
X.~Deng, Y.~Hu, Y.~Deng, S.~Mahadevan, Environmental impact assessment based on
  {D} numbers, Expert Systems with Applications 41~(2) (2014) 635--643.

\bibitem{deng2015d}
X.~Deng, X.~Lu, F.~T. Chan, R.~Sadiq, S.~Mahadevan, Y.~Deng, D-{CFPR}: {D}
  numbers extended consistent fuzzy preference relations, Knowledge-Based
  Systems 73 (2015) 61--68.

\bibitem{xiao2018novel}
F.~Xiao, A novel multi-criteria decision making method for assessing
  health-care waste treatment technologies based on {D} numbers, Engineering
  Applications of Artificial Intelligence 71 (2018) 216--225.

\bibitem{li2018dCC}
X.~Li, X.~Chen, D-intuitionistic hesitant fuzzy sets and their application in
  multiple attribute decision making, Cognitive Computation 10~(3) (2018)
  496--505.

\bibitem{seiti2019developing}
H.~Seiti, A.~Hafezalkotob, S.~E. Najafi, M.~Khalaj, Developing a novel
  risk-based {MCDM} approach based on {D} numbers and fuzzy information axiom
  and its applications in preventive maintenance planning, Applied Soft
  Computing 82 (2019) 105559.

\bibitem{mo2018new}
H.~Mo, Y.~Deng, A new {MADA} methodology based on {D} numbers, International
  Journal of Fuzzy Systems 20~(8) (2018) 2458--2469.

\bibitem{deng2017fuzzy}
X.~Deng, W.~Jiang, Fuzzy risk evaluation in failure mode and effects analysis
  using a {D} numbers based multi-sensor information fusion method, Sensors
  17~(9) (2017) 2086.

\bibitem{deng2019evaluating}
X.~Deng, W.~Jiang, Evaluating green supply chain management practices under
  fuzzy environment: a novel method based on {D} number theory, International
  Journal of Fuzzy Systems 21~(5) (2019) 1389--1402.

\bibitem{IJISTUDNumbers}
X.~Deng, W.~Jiang, A total uncertainty measure for {D} numbers based on belief
  intervals, International Journal of Intelligent Systems 34~(12) (2019)
  3302--3316.

\bibitem{li2016novel}
M.~Li, Y.~Hu, Q.~Zhang, Y.~Deng, A novel distance function of {D} numbers and
  its application in product engineering, Engineering Applications of
  Artificial Intelligence 47 (2016) 61--67.

\end{thebibliography}






% \end{thebibliography}

\end{document}